\documentclass{article}

\usepackage{arxiv}

\usepackage{microtype}
\usepackage{graphicx}
\usepackage{subfigure}
\usepackage{booktabs}

\usepackage[utf8]{inputenc} 
\usepackage[T1]{fontenc}    
\usepackage{hyperref}       
\usepackage{url}            
\usepackage{booktabs}       
\usepackage{amsfonts}       
\usepackage{nicefrac}       
\usepackage{microtype}      
\usepackage{lipsum}
\usepackage{graphicx}

\usepackage{amsmath}
\usepackage{amssymb}
\usepackage{mathtools}
\usepackage{amsthm}
\usepackage{footnote}
\makesavenoteenv{table}
\makesavenoteenv{tabular}

\usepackage[capitalize,noabbrev]{cleveref}

\usepackage{multirow,bigdelim}
\usepackage{tabularx,tabulary}

\theoremstyle{plain}
\newtheorem{theorem}{Theorem}[section]

\newtheorem{lemma}[theorem]{Lemma}
\newtheorem{corollary}[theorem]{Corollary}
\theoremstyle{definition}

\theoremstyle{remark}
\newtheorem{remark}[theorem]{Remark}

\title{Attention Enables Zero Approximation Error}

\author{
 Zhiying Fang \\
  School of Data Science\\
  The Chinese University of Hong Kong, Shenzhen\\
  \texttt{fangzhiying@cuhk.edu.cn} \\
   \And
 Yidong Ouyang \\
  School of Data Science\\
  The Chinese University of Hong Kong, Shenzhen\\
  \texttt{yidongouyang@link.cuhk.edu.cn} \\
  \And
 Ding-Xuan Zhou\\
  School of Data Science and   Department of mathematics\\
  City University of Hong Kong\\
  \texttt{mazhou@cityu.edu.hk} \\
    \And
 Guang Cheng\\
  Department of Statistics\\
  University of California, Los Angeles\\
  \texttt{guangcheng@ucla.edu} \\
}

\begin{document}
\maketitle
\begin{abstract}
Deep learning models have been widely applied in various aspects of daily life. Many variant models based on deep learning structures have achieved even better performances. Attention-based architectures have become almost ubiquitous in deep learning structures. Especially, the transformer model has now defeated the convolutional neural network in image classification tasks to become the most widely used tool. However, the theoretical properties of attention-based models are seldom considered. In this work, we show that with suitable adaptations, the single-head self-attention transformer with a fixed number of transformer encoder blocks and free parameters is able to generate any desired polynomial of the input with no error. The number of transformer encoder blocks is the same as the degree of the target polynomial. Even more exciting, we find that these transformer encoder blocks in this model do not need to be trained. As a direct consequence, we show that the single-head self-attention transformer with increasing numbers of free parameters is universal. These surprising theoretical results clearly explain the outstanding performances of the transformer model and may shed light on future modifications in real applications. We also provide some experiments to verify our theoretical result.
\end{abstract}


\section{Introduction}
By imitating the structure of brain neurons, deep learning models have replaced traditional statistical models in almost every aspect of applications, becoming the most widely used machine learning tools \cite{lecun2015deep,goodfellow2016deep}. Structures of deep learning are also constantly evolving from fully connected networks to many variants such as convolutional networks \cite{krizhevsky2012imagenet}, recurrent networks \cite{mikolov2010recurrent} and the attention-based transformer model \cite{dosovitskiy2020image}. Attention-based architectures were first introduced in the areas of natural language processing, and neural machine translation \cite{bahdanau2014neural,vaswani2017attention,ott2018scaling}, and now an attention-based transformer model has also become state-of-the-art in image classification \cite{dosovitskiy2020image}. However, compared with significant achievements and developments in practical applications, theoretical properties of attention-based transformer models are not well understood.

Let us describe briefly some current theoretical progress of attention-based architectures. The universality of a sequence-to-sequence transformer model is first established in \cite{yun2019transformers}. After that, a sparse attention mechanism, BIGBIRD, is proposed by \cite{zaheer2020big} and the authors further show that the proposed transformer model is universal if its attention structure contains the star graph. Later, \cite{yun2020n} provides a unified framework to analyze sparse transformer models. Recently, \cite{shi2021sparsebert} studies the significance of different positions in the attention matrix during pre-training and shows that diagonal elements in the attention map are the least important compared with other attention positions. From a statistical machine learning point of view, the authors in \cite{gurevych2021rate} propose a classifier based on a transformer model and show that this classifier can circumvent the curse of dimensionality.

The models considered in the above works all contain attention-based transformer encoder blocks. It is worth noting that the biggest difference between a transformer encoder block and a traditional neural network layer is that it introduces an inner product operation, which not only makes its actual performance better but also provides more room for theoretical derivations.

In this paper, we consider the theoretical properties of the single-head self-attention transformer with suitable adaptations. Different from segmenting $x$ into small pieces \cite{dosovitskiy2020image} and capturing local information, we consider a global pre-processing of $x$ and propose a new vector structure of the inputs of transformer encoder blocks. In this structure, in addition to the global information we obtain from data pre-processing, we place a one-hot vector to represent different features through the idea of positional encoding and place a zero vector to store the output values after each transformer encoder block. With such a special design, we can fix all transformer encoder blocks such that no training is needed for them. And it is able to realize the multiplication operation and store values in zero positions. By applying a well-known result in approximation theory \cite{zhou2018deep} stating that any polynomial $Q \in \mathcal{P}_q \left( \mathbb{R}^d\right)$ of degree at most $q$ can be represented by a linear combination of different powers of ridge forms $\xi_k \cdot x$ of $x\in \mathbb{R}^d$, we prove that the proposed model can generate any polynomial of degree $q$ with $q$ transformer encoder blocks and a fixed number of free parameters. As a direct consequence, we show that the proposed model is universal if we let the the number of free parameters and transformer encoder blocks go to infinity. Our theoretical results are also verified by experiments on synthetic data. In summary, the contributions of our work are as follows:
\begin{itemize}
    \item We propose a new pre-processing method that captures global information and a new structure of input vectors of transformer encoder blocks.

    \item With the special structure of input of transformer encoder blocks, we can artificially design all the transformer encoder blocks in a spare way and prove that the single-head self-attention transformer with $q$ transformer encoder blocks and fixed number of free parameters is able to generate any desired polynomial of degree $q$ of the input with no error.

    \item As a direct consequence, we show that the single-head self-attention transformer with increasing numbers of free parameters and transformer encoder block is universal.

    \item We apply our model to noisy regression tasks with synthetic data. Our experiments show that the proposed model performs much better than traditional fully connected neural networks with a comparable number of free parameters.
\end{itemize}

\begin{figure*}[t]
    \label{arc}
	\centering
	\includegraphics[width=.8\textwidth]{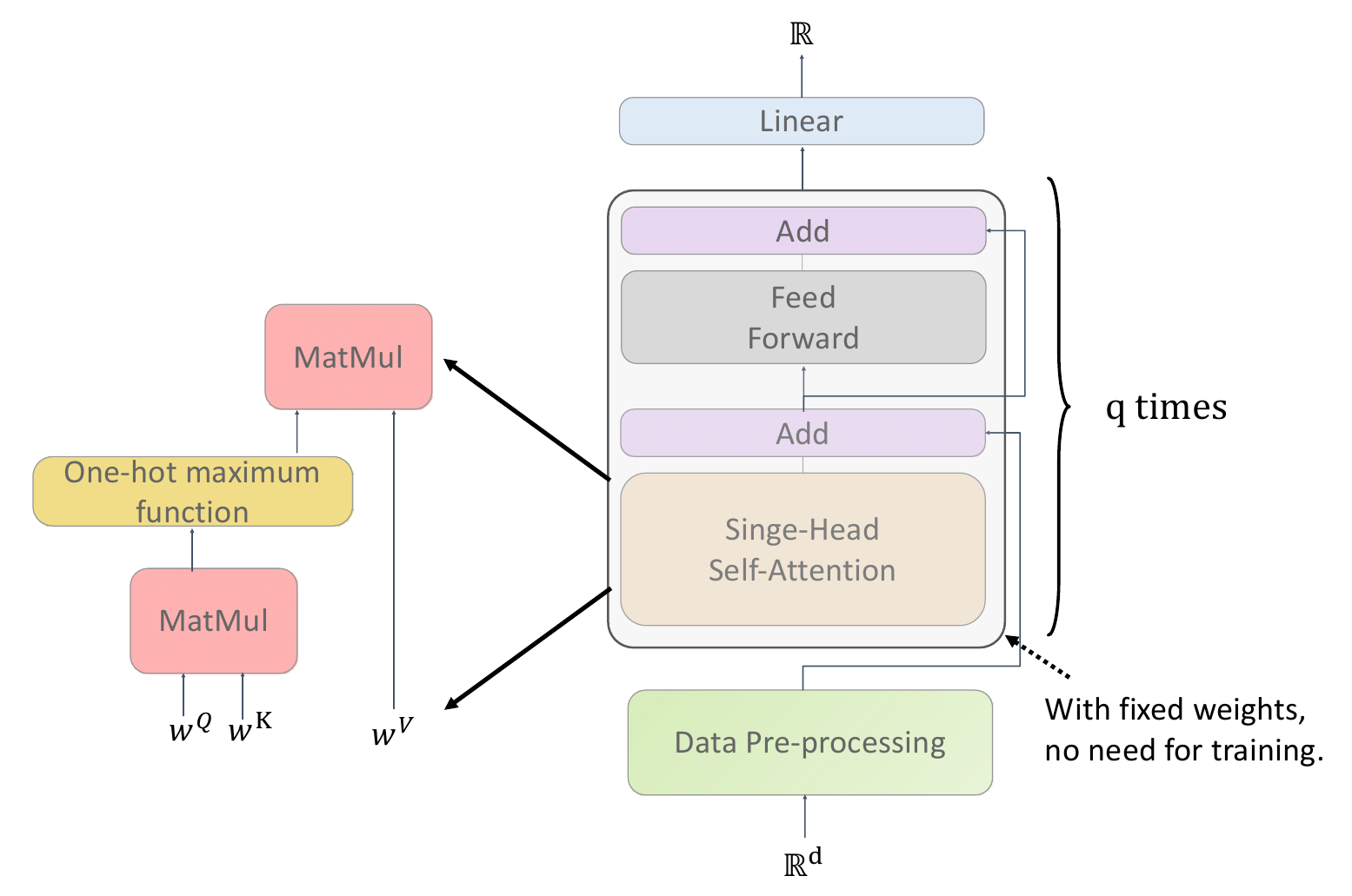}
	\caption{The Architecture of the single-head self-attention transformer. $W^Q, W^K, W^V$ stand for the query matrix, the key matrix, and the value matrix respectively. MatMul stands for the matrix multiplication.}
	\label{fig-motivation}
\end{figure*}

\section{Transformer Structures}
In this section, we formally introduce the single-head self-attention transformer considered in this paper. The overall architecture is shown in Figure \ref{arc}.

\subsection{Data Pre-processing}
\label{processing}
For an input $x \in \mathbb{R}^{d}$ which can be a vector or the concatenation of an image, the usual pre-processing method is to segment it into small pieces and then conduct linear transforms, which can be thought of as extracting local features. However, we propose to directly apply a full matrix $F \in \mathbb{R}^{n \times d}$ to get global features $Fx = t \in \mathbb{R}^{n}$, where $F= [\xi_1, \cdots, \xi_{n} ]^\top$ with $\xi_i \in \mathbb{R}^{d}$ and $\left\|\xi_i\right\|\leq 1$. The matrix $F$ is obtained through the training process. Then we have $n$ global features $t_i = \langle \xi_i , x \rangle$ of the input $x$. Now we introduce the structure of inputs for transformer encoder blocks as follows,
$$z_i = [t_i, \overbrace{0,\cdots,0,\underbrace{1}_{(i+1)-\text{th entry}},0,\cdots,0}^{n}, \overbrace{0,\cdots,0}^q ,1 ]^\top,$$
for $i=1,\cdots,n$. Each one of them is a sparse vector in $\mathbb{R}^{n+q+2}$ and all the $n$ vectors are inputs for the transformer encoder blocks.
As we have covered before, we put a one-hot vector of dimension $n$ inside $z_i$ representing different features $t_i$ of the input $x$ which is similar to the idea of positional encoding. And we also place a $q$ dimensional zero vector to store outputs from each transformer encoder block. At the last position, we place a constant $1$ for the ease of computation in transformer encoder blocks.
We use
$$\mathcal{F}(x) : \mathbb{R}^d \rightarrow \mathbb{R}^{(n+q+2) \times (n)}$$
to denote the above transformation such that $$\mathcal{F}(x) = [z_1, \cdots,z_n].$$
\subsection{Single-Head Self-Attention Transformer Encoder Blocks}
One transformer encoder block contains a self-attention layer and a fully connected layer with a linear transformation. In the self-attention layer, we have one query matrix
$$W^Q \in \mathbb{R}^{(n+1) \times (n+q+2)},$$
one key matrix
$$W^K\in \mathbb{R}^{(n+1) \times (n+q+2)},$$
and one value matrix
$$W^V\in \mathbb{R}^{(n+q+2) \times (n+q+2)}.$$

For every input $z_i$, we calculate the query vector
$$q_i=W^{Q} z_i\in \mathbb{R}^{n+1},$$
the key vector
$$k_i=W^{K} z_i \in \mathbb{R}^{n+1},$$
and the value vector
$$v_i=W^{V} z_i \in \mathbb{R}^{n+q+2}.$$
With all these values, we have $n$ attention vectors
$$\alpha_i= [\langle q_i,k_1\rangle, \cdots ,\langle q_i, k_i \rangle, \cdots,\langle q_i, k_n \rangle]^\top \in \mathbb{R}^{n}.$$
In our proposed model, the softmax function in the self-attention layer is replaced by a one hot maximum function $\hat{m}(\alpha_i): \mathbb{R}^{n} \rightarrow \mathbb{R}^{n}$ which keeps the largest value unchanged and sets the other values to $0$.
We use the notation
$$\mathcal{A}_{W^Q,W^K,W^V} : \mathbb{R}^{(n+q+2) \times n} \rightarrow \mathbb{R}^{(n+q+2) \times n}$$ to denote the mapping of the self-attention layer.
Then the output of the self-attention layer is given by
$$\mathcal{A}_{W^Q,W^K,W^V}(z_1, \cdots, z_n) = [\hat{z}_1, \cdots, \hat{z}_n],$$
where
$$\hat{z}_i = z_i + W^V Z\hat m({\alpha}_i),$$
with $Z = [z_1, \cdots, z_n]$.

The fully connected layer with a linear transformation contains two matrices
$$W_1 \in \mathbb{R}^{2 \times (n+q+2)},$$
and
$$W_2 \in \mathbb{R}^{(n+q+2) \times 2},$$
and two bias vectors $b_1\in \mathbb{R}^{2}$, $b_2\in \mathbb{R}^{n+q+2}$. We use the notation
$$\mathcal{B}_{W_1,W_2,b_1,b_2} : \mathbb{R}^{(n+q+2) \times n} \rightarrow \mathbb{R}^{(n+q+2) \times n}$$
to denote the mapping of the fully connected layer with a linear transformation. Then we have
$$\mathcal{B}_{W_1,W_2,b_1,b_2}(\hat{z}_1, \cdots, \hat{z}_n) = [z'_1, \cdots, z_n'],$$
where
$$z'_i = \hat{z}_i + W_2 \sigma\left(W_1\hat{z}_1+b_1\right) +b_2,$$
and $\sigma$ is the ReLU activation function acting component-wise.

Now we define our single-head self-attention transformer model with $\ell$ transformer encoder blocks as
$$\mathcal{T}^\ell(x) = \mathcal{B}^\ell \circ \mathcal{A}^\ell \circ \cdots \circ \mathcal{B}^1 \circ \mathcal{A}^1 \circ \mathcal{F} (x),$$
where $\mathcal{F}, \mathcal{A}^i, \mathcal{B}^i$ are the mappings defined above.
We further concatenate the output matrix into one vector and apply a linear transformation with a bias term to get our final output, that is,
$$ \mathcal{C}^\ell(x) = \beta \cdot \textbf{concat}\left( \mathcal{T}^\ell(x)\right) + b,$$
with $\beta \in \mathbb{R}^{n(n+q+2)}$ and $b \in \mathbb{R}$. We require the vector $\beta$ to possess a sparse structure which will be shown in the proof. The values in $\beta$ and $b$ are obtained through the training process. The layer normalization is not considered in our model.

\section{Main Results}
In this section, we present our main result showing that the single-head self-attention transformer model can generate any desired polynomial with a fixed number of transformer encoder blocks and free parameters. Before stating our main theorem, we first present two important lemmas. For the following lemma, we construct a sparse single-head self-attention block with fixed design which is able to realize the multiplication operation and store different products in the output vectors simultaneously.
\begin{lemma}\label{main result:one layer}
For all $n$ input vectors in the form of
$$z_i = [t_i, e_i, \overbrace{x_i,y_i,0,\cdots,0}^q ,1 ]^\top \in \mathbb{R}^{(n+q+2) \times 1},$$
with $t_i, x_i, y_i \in \mathbb{R}$ and absolute values bounded by some known constant $M$ for $i=1,\cdots,n$, there exists a sparse single-head self-attention transformer encoder block with fixed matrices $W^Q$, $W^K$, $W^V$, $W_1$, $W_2$ and vectors $b_1$, $b_2$ that can produce output vectors as
$$z'_i = [t_i,e_i, \overbrace{x_i,y_i,-x_iy_i,0,\cdots,0}^q ,1 ]^\top \in \mathbb{R}^{(n+q+2) \times 1},$$
where $e_i$ denotes the one-hot vector of dimension $n$ with value $1$ in the $i$-th position of $e_i$. The softmax function is replaced by one hot maximum function. The number of non-zero entries is $2n+8$.
\end{lemma}

\begin{remark}
The above lemma shows that a fixed single-head self-attention transformer encoder block is able to simultaneously calculate the product of two elements in all $n$ input vectors within the same two entries and store the negative value in the same $0$ positions. Since the construction is fixed, these transformer encoder blocks in the whole model do not need to be trained.
\end{remark}
Now we introduce a well-known result in approximation theory showing that any polynomial function $Q \in \mathcal{P}_q \left( \mathbb{R}^d\right)$ of degree at most $q$ can be represented by a linear combination of different powers of ridge forms $\xi_k \cdot x$ of $x\in \mathbb{R}^d$. The following lemma is first presented and proved in \cite{zhou2018deep} and also plays an important role in the analysis of deep convolutional neural networks \cite{zhou2020universality,mao2021theory}.
\begin{lemma}\label{lemma: polynomial}
Let $d \in \mathbb{N}$ and $q \in \mathbb{N}$. Then there exists a set $\left\{ \xi_k\right\}_{k=1}^{n_q} \subset \left\{ \xi \in \mathbb{R}^d : \left\|\xi\right\|=1\right\}$ of vectors with $\ell_2-$norm $1$ such that for any $Q \in \mathcal{P}_q \left( \mathbb{R}^d\right)$ we can find a set of coefficients $\left\{\beta_{k, s}: k=1,\cdots, n_q, s = 1,\cdots, q\right \} \subset \mathbb{R}$ such that
\begin{equation}
\begin{aligned}
Q(x) = Q(0) + \sum_{k=1}^{n_q} \sum_{s=1}^q \beta_{k,s} \left( \xi_k \cdot x\right)^s, ~~~x \in \mathbb{R}^d,
\end{aligned}
\end{equation}
where $n_q = \binom{d-1+q}{q}$ is the dimension of $\mathcal{P}^h_q(\mathbb{R}^d),$ the space of homogeneous polynomials on $\mathbb{R}^d$ of degree $q$.
\end{lemma}
\begin{remark}
The above lemma shows that any polynomial $Q \in \mathcal{P}_q \left( \mathbb{R}^d\right)$ can be uniquely determined by $Q(0)$, $\beta_{k,s}$ and $\xi_k$. So by applying the above lemma, we can perfectly reproduce any polynomial with proper construction.
\end{remark}

Now we are ready to state our main result on the single-head self-attention transform model.

\begin{theorem}\label{main result:polynomial}
Let $B>0$ and $q \in \mathbb{N}$. For any polynomial function $Q\in \mathcal{P}_q (\mathbb{R}^d)$ of degree at most $q$, there exist a single-head self-attention transformer model with $q$ transformer encoder blocks such that the output function $\mathcal{C}^q$ equals $Q$ on $\left\{x\in \mathbb{R}^d: \left\|x\right\| \leq B\right\}$
$$\mathcal{C}^q(x) = Q(x),~~\forall \left\|x\right\| \leq B.$$
The number of free parameters is less then $d^{q+1}+ qd^q +1$ which comes from $F$, $\beta$ and $b$. The number of non-zero entries in this model is less than $d^{q+1} +3qd^q + 8q  +1.$
\end{theorem}
\begin{remark}
The above theorem shows a very strong property of the self-attention transformer model that it can generate any desired polynomial with a finite number of free parameters. As we can see, the degree of the polynomial is reflected in the number of transformer encoder blocks, showing that the more blocks the transformer has, the more complex polynomial it can represent. Clearly, this result outperforms that of the other classical deep learning models without attention-based structure in at least two aspects. First, since the linear combination of the output units of traditional ReLU neural networks is only a piece-wise linear function of the input, no matter how many finite layers and free parameters, it can never produce a polynomial of the input with no error. Second, the transformer encoder blocks in our construction only serves as the realization of the multiplication operation. The non-zero values are all pre-designed constants, so no training is needed for these blocks. We only need to train free parameters in $F$, $\beta$ and $b$.
\end{remark}

As a direct consequence of the above result, the proposed single-head self-attention transform model is universal.
\begin{corollary}
Let $d\in \mathbb{N}$ and $q \in \mathbb{N}$. For any bounded continuous function $f$ on $[0,1]^d$, there exists a single-head self-attention transformer with increasing numbers of free parameters and transformer encoder blocks such that
$$\lim_{q \rightarrow \infty} \left\|\mathcal{C}^q - f\right\|_{C([0,1]^d)} = 0$$
\end{corollary}

The above result is a simple application of the denseness of the polynomial set, which shows that the transformer model discussed in our paper is universal if we let the number of free parameters and transformer encoder blocks go to infinity.

\section{Comparison and Discussion}
In this section, we compare our work with some existing theoretical results on the transformer model \cite{yun2019transformers,yun2020n,zaheer2020big,shi2021sparsebert}. Since these works use similar methods to those in \cite{yun2019transformers}, we focus on the theoretical contributions of this paper.

In \cite{yun2019transformers}, the authors show that transformer models are universal approximators of continuous sequence-to-sequence functions with compact support with trainable positional encoding. The notion of contextual mappings is also formalized, and it is shown that the attention layers can compute contextual mappings, where each unique context is mapped to a unique vector.

The universality result is achieved in three key steps: \textbf{Step 1}. Approximate continuous permutation equivariant functions with piece-wise constant functions $\bar{\mathcal{F}}_{PE}(\delta)$. \textbf{Step 2}. Approximate $\bar{\mathcal{F}}_{PE}(\delta)$ with modified Transformers $\bar{\mathcal{T}}$. \textbf{Step 3}. Approximate modified Transformers $\bar{\mathcal{T}}$ with original Transformers $\mathcal{T}$.

In order to express the above steps more clearly, we show the idea of proof as follows. For an input $X \in\mathbb{R}^{d\times n}$, the authors first use a series of feed-forward layers that can quantize $X$ to an element $L$ on the extended grid $\mathbb{G}^+_\delta := \left\{-\delta^{-nd},0,\delta,\cdots,1-\delta\right\}^{d\times n}$. Activation functions that are applied to these layers are piece-wise linear functions with at most three pieces, and at least one piece is constant. Then, the authors use a series of self-attention layers in the modified transformer network to implement a contextual mapping $q(L)$. After that, a series of feed-forward layers in the modified transformer network can map elements of the contextual embedding $q(L)$ to create a desired approximator $\bar{g}$ of the piece-wise constant function $\bar{f} \in \bar{\mathcal{F}}_{PE}(\delta)$ which is the approximator of the target function.

We would like to address major differences between our work and theirs. First, the output functions are different. In the above work, the goal is to approximate a continuous function defined from $\mathbb{R}^{n\times d}$ to $\mathbb{R}^{n\times d}$, which focuses on sequence-to-sequence functions. In our setting, we use the linear combination of the units in the last layer as our output, which focuses on regression and classification tasks. Second, the two structures we consider are slightly different. The self-attention layers and feed-forward layers in their transformer model are set in an alternate manner. Although this may explain the different functions of different types of layer, it changes the structure of transformer model in real applications. In our setting, we guarantee the integrity of transformer encoder blocks and analyze each transformer encoder block as a whole. Last but not least, the ultimate goals and core ideas of the theoretical analysis of our two papers are different. Because the inner product operation is the biggest difference between the attention layer and the traditional network layer, we focus on this special structure for analysis. We find that if we can make good use of this inner product structure, then from the perspective of theoretical analysis, we do not have to think about approximation but can directly generate the function we want. And the exact construction only requires a finite number of free parameters with fixed transformer encoder blocks. This shows the different thinking in our theory and distinguishes our method from using piece-wise functions to approximate target functions.

\begin{figure*}[t]
	\centering
  \subfigure[ATTENTION with MSE 0.02. ]{\includegraphics[width=.30\textwidth]{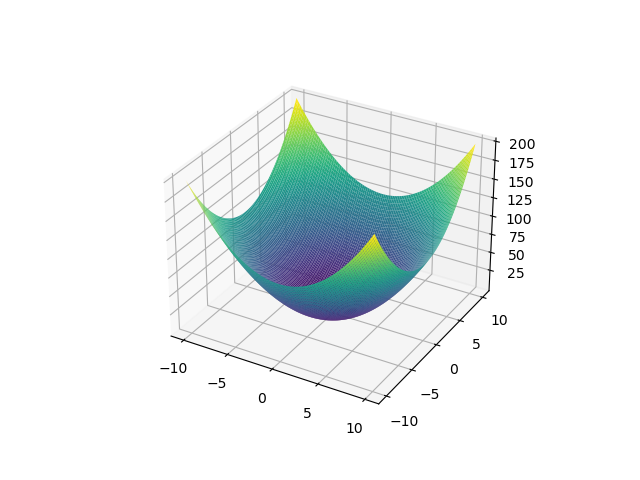}}
\subfigure[NN$_{depth}$ with MSE 134.23.]{	\includegraphics[width=.30\textwidth]{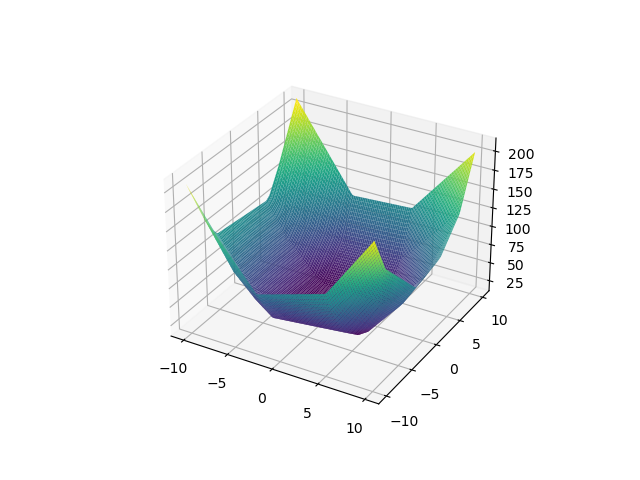}}
\subfigure[NN$_{width}$ with MSE 10237.27.]{\includegraphics[width=.30\textwidth]{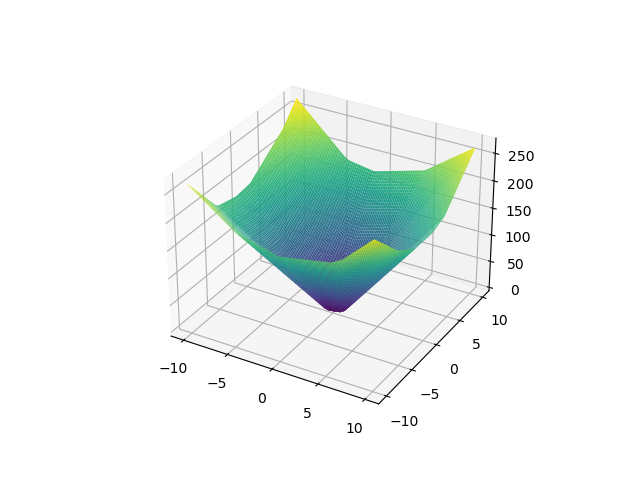}}
	\caption{For the target polynomial $f^*_1$, the above 3-D surface plots are output functions of three different models after the training process. ATTENTION stands for our single-head self-attention transformer model, while NN$_{depth}$ and NN$_{width}$ stand for fully connected neural networks illustrated in experimental setting. MSE stands for the Mean Squared Error evaluated at testing data.}
	\label{f1}
\end{figure*}

\section{Experiments: Learning Polynomial Functions}
\label{experiments}
In this section, we verify our main results and demonstrate the superiority of our single-head self-attention transformer model by conducting experiments on two groups of synthetic data.
\paragraph{Target functions}
For these two experiments, we consider the noisy regression task
$$y =  f^*(\textbf{x}) + \epsilon,$$
where $f^*$ is the target polynomial and $\epsilon$ is the standard normal noise.

For the first experiment, in order to visualize the advantages of our proposed model, we consider a simple polynomial,
\begin{equation*}
   \begin{aligned}
   f^*_1(\textbf{x}) = x_1^2 + x_2^2,
   \end{aligned}
\end{equation*}
which satisfies $d=2$ and $q=2$.

For the second experiment, to show the strong expressiveness of our model, we consider a complicated polynomial
\begin{equation*}
   \begin{aligned}
   &f^*_2(\textbf{x}) =\\
   &x_1^5+3x_2^4+2x_3^3+5x_3x_4+3x_5^2+2x_6x_7x_8+2x_9,
   \end{aligned}
\end{equation*}
which satisfies $d=10$ and $q=5$.
\paragraph{Data generating process}
For the target function $f^*_1$, we generate 10000 i.i.d. sample $\textbf{x}$ from a multivariate Gaussian distribution $\mathcal{N}(\textbf{0}, \Sigma_1)$ with $\Sigma_1 = \text{diag}(100,100)$. We randomly choose 9000 of them for training and 1000 data for testing. \\

For the target function $f^*_2$, we generate 50000 i.i.d. sample $\textbf{x}$ from a multivariate Gaussian distribution $\mathcal{N}(\textbf{0}, \Sigma_2)$ with $\Sigma_2 = \text{diag}(1,\cdots,1) \in \mathbb{R}^{10\times 10}$. We randomly choose 45000 of them for training and 5000 data for testing.

\paragraph{Experimental setting}
To demonstrate the power of attention-based structures, we compare our proposed model with two types of ReLU fully connected neural networks with a comparable number of free parameters. Since for a polynomial $Q$ of degree $q$, our proposed model has one linear transformation with matrix $F\in\mathbb{R}^{n_q \times d}$ and $q$ transformer encoder blocks, we use NN$_{depth}$ to denote the fully connected network with $q+1$ layers and we use NN$_{width}$ to denote the shallow net with $n_q$ units in the hidden layer. For these two fully connected networks, we use the same way as our proposed model to generate output value, which is the linear combination of units in the last layer with a bias term. The detailed architectures can be found in \ref{app_exp}. \\

In all the experiments, we use SGD optimizer with one cycle learning rate \cite{Smith2019SuperconvergenceVF}, with an initial learning rate 0.0001 and maximum learning rate 0.001. For the polynomial $f_1^*$, we train three models 600 epochs with batch size 5000, and for the polynomial $f_2^*$, we train three models 2000 epochs with batch size 25000. The gradient clipping is used for all three models to avoid gradients exploding at the beginning of training.

\begin{figure}[h]

	\centering
	\includegraphics[width=.45\textwidth]{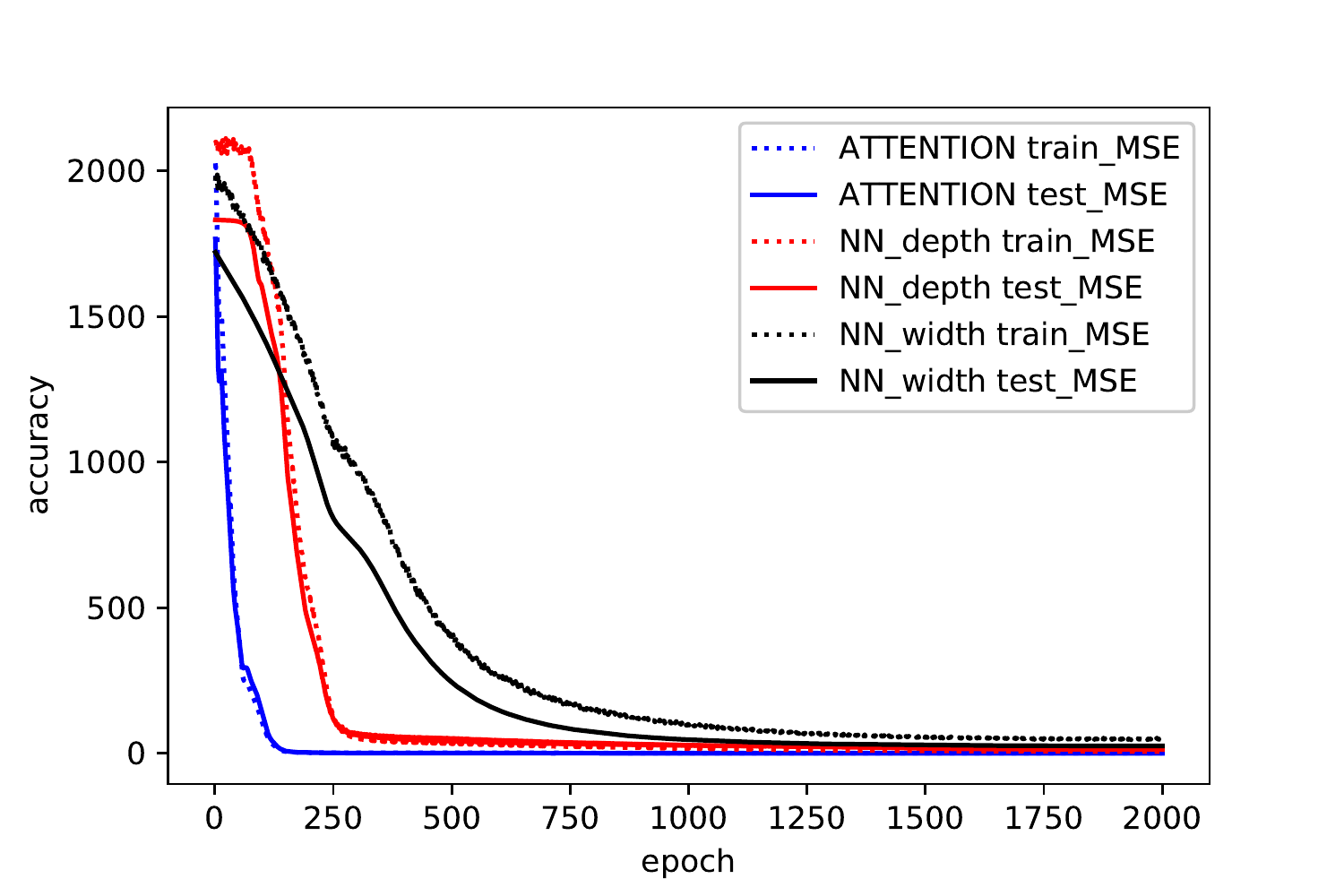}
	\caption{A comparison of the convergence speed and generalization gap between our single-head self-attention model and two types of fully connected neural networks.}
	\label{fig-train}
\end{figure}
\paragraph{Experimental results} For the target polynomial $f_1^*$, Figure \ref{f1} demonstrates the strong power of learning polynomials of our proposed model. With only 41 free parameters, our single-head self-attention transformer can perfectly capture the target function by using noisy data. Due to the nature of piece-wise linear output function, both two types of fully connected neural networks obviously can not achieve comparable results with very few parameters. \\

For the target function $f_2^*$, Table \ref{t1} and Figure \ref{fig-train} also demonstrate the superior ability of our model to learn a complicated polynomial. Our single-head self-attention transformer is the only one that can fit the ground truth function exactly with good convergence speed. Moreover, our model has a much better generalization power than both two types of fully connected neural networks with a similar number of free parameters.
\begin{table}[t]
\caption{A comparison of three models learning $f_2^*$. MSE$_{Tr}$ and MSE$_{Te}$ stand for the mean-squared error of the training data and the testing data after 2000 epochs training, respectively. We say that a model achieves convergence if the absolute difference of MSE$_{Tr}$ of two consecutive epochs is less than 0.01. \# EPOCHS stands for the number of epochs the model used before achieving convergence, and RUN TIME represents the corresponding running time of the training process. }
\label{t1}
\vskip 0.15in
\begin{center}
\begin{small}
\begin{sc}
\begin{tabular}{lcccc}
\toprule & MSE$_{Tr}$ & MSE$_{Te}$ & \# epochs  & Run time\footnote{GPU * min on NVIDIA A100 Tensor Core GPU.} \\

Attention& \textbf{0.938} & \textbf{0.109} & \textbf{212} & \textbf{1.9}\\
NN$_{depth}$ & 5.884 & 103.916 & 956& 7.6\\
NN$_{width}$ & 50.282 & 35.662 & 329 & 2.4\\
\bottomrule
\end{tabular}
\end{sc}
\end{small}
\end{center}
\vskip -0.1in
\end{table}

\section{Proof of Main Results}

\begin{proof}[Proof of Lemma \ref{main result:one layer}]
We present explicit constructions of matrices and biases in single-head self-attention transformer encoder block. We let $W^Q \in \mathbb{R}^{(1+n) \times (2+n+q)}$ as follows,
\begin{equation*}
W^{Q}=\left[\begin{array}{cccccccc}
{0} & {\cdots} & {0} & {1} & {0} & {\cdots} & {0}& {0} \\
{0} & {2M^2} & {0} & {0} & {\cdots} & {\cdots} & {0}& {0} \\
{0} & {0} & {2M^2} & {0} & {\ddots} & {\ddots} & {0}& {0} \\
{0} & {0} & {0} & {\ddots} & {\ddots} & {\cdots} & {0}& {0} \\
{0} & {0} & {0} & {0} & {2M^2} & {0} & {\cdots}& {0}
\end{array}\right],
\end{equation*}
where the constant $1$ in the first row is in the $(n+2)-$th column. And we set $W^Q_{(t,t)} = 2M^2$ for $t=2,\cdots,n+1$ and all the other elements $0$. Since the inputs are in the form of
$$z_i = [t_i, e_i, \overbrace{x_i,y_i,0,\cdots,0}^q ,1 ]^\top \in \mathbb{R}^{(n+q+2) \times 1},$$
where $e_i$ denotes the one-hot vector of dimension $n$ with value $1$ in the $i$-th position of $e_i$. Then we know that $q_i\in \mathbb{R}^{(1+n)\times 1}$ is as follows
$$q_i=W^{Q} z_i = [x_i, \underbrace{0, \cdots, 0, \overbrace{2M^2}^{\text{$(i+1)$-th entry}}, 0, \cdots,0}_{n}]^\top .$$
We let $W^{K} \in \mathbb{R}^{(1+n) \times (2+n+q)}$ as follows
\begin{equation*}
W^{K}=\left[\begin{array}{cccccccc}
{0} & {\cdots} & {0} & {1} & {\cdots} & {\cdots} & {0}& {0} \\
{0} & {1} & {0} & {0} & {\cdots} & {\cdots} & {0}& {0} \\
{0} & {0} & {1} & {0} & {\ddots} & {\ddots} & {0}& {\vdots} \\
{0} & {0} & {0} & {\ddots} & {\ddots} & {\cdots} & {0}& {0} \\
{0} & {0} & {0} & {0} & {1} & {0} & {\cdots}& {0}
\end{array}\right].
\end{equation*}
where the constant $1$ in the first row is in the $(n+3)-$th column. And we set $W^K_{(t,t)} = 1$ for $t=2,\cdots,n+1$ and other elements $0$. Then we have $k_i \in \mathbb{R}^{(1+n)\times 1}$ as
$$k_i=W^{Q} z_i = [y_i, \underbrace{0, \cdots, 0, \overbrace{1}^{\text{$(i+1)$-th entry}}, 0, \cdots,0}_{n}]^\top .$$
We can easily find that for each $i$, if $j = i$, then $\langle q_i, k_j\rangle = x_iy_i+2M^2$. And if  $j \neq i$, then $\langle q_i, k_j\rangle = x_iy_j$. By the condition $\left|x_i\right| < M$ and $\left|y_i\right| < M$, clearly we have $x_iy_i+2B^2 > x_iy_j $.

Then the attention vector $\alpha_i$ is
$$\alpha_i= [\langle q_i,k_1\rangle, \cdots ,\langle q_i, k_i \rangle, \cdots,\langle q_i, k_n \rangle]^\top \in \mathbb{R}^{n\times 1}.$$
Since we apply the one hot maximum function to $\alpha_1$, then by the construction we have
$$\hat{\alpha}_i= [0, \cdots ,0, \overbrace{x_i y_i + 2B^2}^{i-\text{th entry}}, 0,\cdots,0]^\top \in \mathbb{R}^{n\times 1}.$$
For the matrix $W^V \in \mathbb{R}^{(n+q+2)\times(n+q+2)}$, we set
$$
W^V_{i,j}= \left\{
\begin{array}{rcl}
1,& &i=n+4, j=n+q+2,\\
0,& &\text{others.}
\end{array}
\right.
$$
Then we know that for $i=1,\cdots,n$,
$$W^V z_i = [0,\cdots,0, \overbrace{1}^{(n+4)- \text{th entry}}, 0,\cdots,0 ]^\top.$$
By the equation
$$\hat{z}_i = z_i + W^V Z\hat{\alpha}_i,$$
we know that the outputs $z_i \in \mathbb{R}^{(n+q+2) \times 1}$ of self-attention layer are
$$\hat{z}_i = [t_i, e_i, \overbrace{x_i,y_i,x_iy_i+2M^2,0,\cdots,0}^q ,1 ]^\top ,$$
where $e_i$ denotes the one-hot vector of dimension $n$ with value $1$ in the $i$-th position of $e_i$.

Now we construct the fully connected layer in the transformer.
For $W_1 \in \mathbb{R}^{2 \times (2+n+q)}$, we let
$$
W_{1,(i,j)}= \left\{
\begin{array}{rcl}
1,& &i=1, j=n+4,\\
-1,& &i=2, j=n+4,\\
0,& &\text{others.}
\end{array}
\right.
$$
and $b_1= [0,0]^\top$. Then we have
$$\sigma(W_1 z_i +b_1) = [\sigma(x_iy_i+2M^2),\sigma(-x_iy_i - 2M^2)]^\top.$$
For $W_2 \in \mathbb{R}^{(n+q+2) \times 2}$, we let
$$
W_{2,(i,j)}= \left\{
\begin{array}{rcl}
-2,& &i=n+4, j=1,\\
2,& &i=n+4, j=2,\\
0,& &\text{others.}
\end{array}
\right.
$$
And we let $b_2 \in \mathbb{R}^{(n+q+2) \times 1}$ to be
$$
b_{2,(i)}= \left\{
\begin{array}{rcl}
2M^2,& &i=n+4,\\
0,& &\text{others.}
\end{array}
\right.
$$
Then by
$$z'_i = \hat{z}_i + W_2 \sigma\left(W_1\hat{z_1}+b_1\right) +b_2,$$
we have
$$z'_i = [t_i,e_i, \overbrace{x_i,y_i,-x_iy_i,0,\cdots,0}^q ,1 ]^\top \in \mathbb{R}^{(n+q+2) \times 1}.$$
Since we assume that $M$ is known, we do not have any free parameter in this construction. It is easy to see that the number of non-zero entry is $2n+8$. This finishes the proof.
\end{proof}
Now we are ready to prove Theorem \ref{main result:polynomial}.
\begin{proof}[Proof of Theorem \ref{main result:polynomial}]
To prove our main result on polynomial generation, we first apply Lemma \ref{lemma: polynomial}. Since the matrix $F \in \mathbb{R}^{n_q\times d}$ can be obtained by training, we set $F= [\xi_1, \cdots, \xi_{n_q} ]^\top$ and let $\xi_i$ to be those vectors we need in Lemma \ref{lemma: polynomial} for $i=1,\cdots,n_q$. Then we know that the inputs for the transformer encoder blocks are
$$z_i = [\xi_i \cdot x, \overbrace{0,\cdots,0,\underbrace{1}_{(i+1)-\text{entry}},0,\cdots,0}^{n}, \overbrace{0,\cdots,0}^q ,1 ]^\top,$$
for $i=1,\cdots,n_q$. Then we only need to apply Lemma \ref{main result:one layer} $q$ times with suitable adjustments of the position of non-zero entries to make sure that the product of two elements in vectors are saved in a right entry.

For the first transformer encoder block, we calculate the product of $\xi_i \cdot x$ and $1$ and place $-\xi_i \cdot x$ it in the $(n_q+2)-$th entry. Since we know that $\left\|\xi_i\right\|=1$, if we further assume that $\left\|x\right\| < B$, then we have $\left|\xi_i \cdot x\right| \leq B$. Then we only need to set $M=B$ in Lemma \ref{main result:one layer} and the output vectors are
$$z_i = [\xi_i \cdot x, e_i, \overbrace{-\xi_i \cdot x,\cdots,0}^q ,1 ]^\top,$$
where $e_i$ denotes the one-hot vector of dimension $n$ with value $1$ in the $i$-th position of $e_i$.
For the second transformer encoder block, we calculate the product of $\xi_i \cdot x$ and $-\xi_i \cdot x$ to get $\left(\xi_i \cdot x\right)^2$ and place it in the $(n_q+3)-$th entry. We set $M=B$ in Lemma \ref{main result:one layer} and the output vectors are
$$z_i = [\xi_i \cdot x, e_i, \overbrace{-\xi_i \cdot x,\left(\xi_i \cdot x\right)^2,\cdots 0}^q ,1 ]^\top,$$
Without loss of generality, we set $q$ to be odd. For the $i$-th block with $i=3,\cdots,q$, we set $M=B^{i-1}$. Then after $q$ transformer encoder blocks, the outputs are
$$z_i = [t_i, e_i, \overbrace{-t_i,t_i^2,\cdots,-t_i^q}^q ,1 ]^\top \in \mathbb{R}^{(n+q+2) \times 1},$$
where $t_i = \xi_i \cdot x$ for $i=1,\cdots,n_q$. Now we have different powers of $\xi_i \cdot x$ for $i=1,\cdots,n_q$. Then we only need to set elements of $\beta$ as those $\beta_{k,s}$ we need in Lemma \ref{lemma: polynomial} and $b = Q(0)$ to generate the polynomial $Q$ we want.

Since we assume that $B$ is known, then there is no free parameter in transformer encoder blocks. The free parameters in our model all come from $F$, $\beta$ and $b$. By $n_q = \binom{d-1+q}{q}$, it is easy to see that $n_q \leq d^q$. The number of free parameters in $F$ is less then $d^{q+1}$. Since for each $z_i$, we only need $q$ non-zero entries in $\beta$, the number of free parameters in $\beta$ is less then $qd^q$. So the total number of free parameters is less than $d^{q+1}+ qd^q +1$.

The number of non zero entries in this model is those in $F$, $W^K$, $W^Q$, $W^V$, $W_1$, $W_2$, $b_1$ $b_2$ in each block and in $\beta$, $b$. It can be calculated easily to know the number of non zero entries is less than $d^{q+1} +3qd^q + 8q  +1.$

This finishes the proof.
\end{proof}

\section{Conclusion}
In this paper, we introduced a single-head self-attention transformer model and showed that any polynomial can be generated exactly by an output function of such a model with the number of transformer encoder blocks equal to the degree of the polynomial. The transformer encoder blocks in this model do not need to be trained.

In the future, many research directions will be very attractive. First of all, our core idea is different from traditional one of approximation, and through the appropriate adjustment of the transformer model, a completely new theoretical result is presented. Also, in our structure, the transformer encoder blocks are completely fixed, it is of great interest to check our results in real applications to see whether these adaptations can indeed bring benefits. Second, we have obtained such exciting theoretical results by considering only single-head self-attention structure. We can consider whether the multi-head structure can lead to more surprising conclusions. Last but not least, it is of great interest to consider this model under the setting of statistical machine learning. As we can see in our experiments, as long as the number of free parameters meets the theoretical requirement, our model can not only learn the objective function well, but also has a much stronger generalization ability than other models. And as far as we are concerned, this is the first deep learning model which is capable of reaching zero approximation error for certain function class. We will investigate how such a model affects convergence rates for regression or classification problems if the target function is a polynomial of the input and we will verify whether convergences rates now only depend on the complexity of the proposed model.

\clearpage
\nocite{langley00}

\bibliography{thesisbib_2}
\bibliographystyle{plain}

\newpage
\appendix
\onecolumn
\section{Experimental Details}
In this section we describe the additional details our experiments.

\subsection{Model architectures}
\label{app_exp}
Table \ref{arc-f1} and \ref{arc-f2} illustrate the architecture of two types of ReLU fully connected neural networks with a comparable number of free parameters used in Section \ref{experiments}. The NN$_{width}$ has the same kind of linear transformation from $\mathcal{R}^d \to \mathcal{R}^{n_q}$ as our single-head self-attention transformer, while the NN$_{depth}$ has the same hidden layer $q+1$ as our single-head self-attention transformer.

\begin{table}[h!]
  \caption{The architecture of NN$_{width}$ and NN$_{depth}$ for the target polynomial $f^*_1$.}
\vskip 0.15in
\begin{center}
\begin{small}
\begin{sc}
  \begin{tabular}{l l l l}
\toprule
    \multicolumn{1}{l}{Layer} & \multicolumn{1}{l}{NN\_{width}} & NN\_{depth} \\
\midrule
    1     & \multicolumn{1}{l}{Linear(in=2,out=10)} & Linear(in=2,out=4) \\
\midrule
    2     & \multicolumn{1}{l}{Relu} & Relu \\
\midrule
    3     & \multicolumn{1}{l}{Linear(in=10,out=1)} & Linear(in=4,out=4) & \rdelim\}{2}{1mm}[$\times $ 2]       \\
\midrule
    4     &       & Relu \\
\midrule
$\cdots$ & & \\
\midrule
    7    &       & Linear(in=4,out=1) \\
\bottomrule
  \end{tabular}
  \label{arc-f1}
\end{sc}
\end{small}
\end{center}
\vskip -0.1in
\end{table}

\begin{table}[h!]
  \caption{The architecture of NN$_{width}$ and NN$_{depth}$ for the target polynomial $f^*_2$.}
\vskip 0.15in
\begin{center}
\begin{small}
\begin{sc}
  \begin{tabular}{l l l l}
\toprule
    \multicolumn{1}{l}{Layer} & \multicolumn{1}{l}{NN\_{width}} & NN\_{depth} \\
\midrule
    1     & \multicolumn{1}{l}{Linear(in=10,out=4368)} & Linear(in=10,out=120) \\
\midrule
    2     & \multicolumn{1}{l}{Relu} & Relu \\
\midrule
    3     & \multicolumn{1}{l}{Linear(in=4368,out=1)} & Linear(in=120,out=120) & \rdelim\}{2}{1mm}[$\times $ 5]         \\
\midrule
    4     &       & Relu \\
\midrule
$\cdots$ & & \\
\midrule
    13    &       & Linear(in=120,out=1) \\
\bottomrule
  \end{tabular}
  \label{arc-f2}
\end{sc}
\end{small}
\end{center}
\vskip -0.1in
\end{table}
\end{document}